\documentclass{article}
\usepackage[a4paper,margin=1in]{geometry}
\usepackage{times}
\usepackage{authblk}
\usepackage{natbib}

\usepackage[explicit]{titlesec}
\titleformat{\section}
  {\normalfont\large\scshape\raggedright}{\thesection}{1em}{#1}
\titleformat{\subsection}
  {\normalfont\normalsize\scshape\raggedright}{\thesubsection}{1em}{#1}
\titleformat{\subsubsection}
  {\normalfont\normalsize\scshape\raggedright}{\thesubsubsection}{1em}{#1}

\titlespacing*{\section}      {0pt}{2.0ex plus 0.5ex minus 0.2ex}{1.5ex plus 0.3ex minus 0.2ex}
\titlespacing*{\subsection}   {0pt}{1.8ex plus 0.5ex minus 0.2ex}{0.8ex plus 0.2ex}
\titlespacing*{\subsubsection}{0pt}{1.5ex plus 0.5ex minus 0.2ex}{0.5ex plus 0.2ex}

\usepackage{hyperref}
\usepackage{url}
\usepackage{amsmath}
\usepackage{mathrsfs}

\usepackage{amsthm}
\usepackage{enumitem}

\setlist{topsep=4pt plus 1pt minus 2pt,
         partopsep=1pt plus 0.5pt minus 0.5pt,
         itemsep=2pt plus 1pt minus 0.5pt,
         parsep=2pt plus 1pt minus 0.5pt}
\usepackage{tikz}
\usetikzlibrary{arrows.meta,positioning,calc}

\usepackage{amsmath,amsfonts,bm}

\def\eqref#1{equation~\ref{#1}}

\def\1{\bm{1}}

\DeclareMathAlphabet{\mathsfit}{\encodingdefault}{\sfdefault}{m}{sl}
\SetMathAlphabet{\mathsfit}{bold}{\encodingdefault}{\sfdefault}{bx}{n}

\newcommand{\E}{\mathbb{E}}

\newcommand{\R}{\mathbb{R}}

\newtheorem{theorem}{Theorem}[section]
\newtheorem{lemma}[theorem]{Lemma}
\newtheorem{remark}[theorem]{Remark}
\newcommand{\cY}{\mathcal Y}
\newtheorem{proposition}[theorem]{Proposition}
\newtheorem{corollary}[theorem]{Corollary}

\theoremstyle{definition}
\newtheorem{definition}[theorem]{Definition}
\newtheorem{example}[theorem]{Example}

\newtheorem{assumption}[theorem]{Assumption}
\theoremstyle{remark}

\newcommand{\inner}[2]{\langle #1, #2 \rangle}
\newif\ifhedwigcomments
\hedwigcommentstrue

\newif\iflukascomments
\lukascommentsfalse

\newcommand{\cX}{\mathcal X}
\newcommand{\cZ}{\mathcal Z}

\newcommand{\cD}{\mathcal D}

\DeclareMathOperator{\SO}{SO}

\newcommand{\PR}{\mathbb{P}}
\newcommand{\B}{\mathscr{B}}

\newcommand{\A}{\mathscr{A}}
\newcommand{\cT}{\mathcal{T}}
\renewcommand{\eqref}[1]{Equation~\textup{\ref{#1}}}

\title{Latent Process Generator Matching}
\date{\today}
\newif\ifhncomments
\hncommentstrue
\newif\ifbencomments
\bencommentstrue

\bencommentstrue

\author{Lukas Billera\thanks{Equal contribution. Correspondence: \texttt{\{lukas.billera, hedwig.nordlinder, benjamin.murrell\}@ki.se}}, Hedwig Nora Nordlinder$^*$, Ben Murrell.}
\affil{Department of Microbiology, Tumor and Cell Biology, Karolinska Institutet, Stockholm, Sweden}

\begin{document}

\maketitle
\begin{abstract}
Many recent flow-matching and diffusion-style generative models rely on auxiliary stochastic dynamics during training: a richer process is simulated to define conditional targets, but the auxiliary state is either intractable to sample at generation time or simply not part of the desired output. Existing Generator Matching theory formalises conditioning on static latent random variables, and several recent papers prove special cases of projection results for particular augmented-state constructions. We introduce latent process generator matching, a general framework that treats the observed generative state as a deterministic image $X_t=\Phi(Y_t)$ of a tractable Markov process $Y_t$. We show that in this setting one may learn the generator of a stochastic process on the image space which has the same one-time marginal distributions as the projected process. This generalizes and subsumes the discrete latent process results from the literature, and extends Generator Matching from static latent variables to a rich family of time-dependent latent conditional processes. 

\end{abstract}
\section{Introduction}
In recent extensions of flow matching and diffusion approaches to generative modelling, one constructs a Markov process on an extended state space \(\mathcal{X} \times \mathcal{Z}\) representing the conditional paths of a generative model, but at generation time one wishes to simulate a process on \(\mathcal{X}\) alone whose one-time marginals coincide with those of the \(\mathcal{X}\)-component of the joint process. This is the case for \citet{ifriqi2025flowceptiontemporallyexpansiveflow}, \citet{nguyen2025oneflowconcurrentmixedmodalinterleaved}, and \citet{havasi2025editflowsflowmatching}. A related situation arises when an auxiliary process is introduced to aid training but modelling its dynamics at generation time is unnecessary or difficult, as in \citet{billera2025branchingflowsdiscretecontinuous} and \citet{kim2025anyorderflexiblelengthmasked}. In each of these works, the projection result and its associated loss are derived on a case-by-case basis, and all theorems are restricted to marginalization over a discrete component of the extended state space. We introduce a general framework that removes these restrictions: given a time-inhomogeneous Feller process \((Y_t)_{0 \leq t \leq 1}\) on an arbitrary state space \(\mathcal{Y}\) and a map \(\Phi \colon \mathcal{Y} \to \mathcal{X}\), one may learn a linear parametrisation of the generator of a Feller process on \(\mathcal{X}\) whose one-time marginals coincide with those of \((\Phi(Y_t))_{0 \leq t \leq 1}\). For $\mathcal Y = \mathcal X \times \mathcal Z$ and $\Phi$ the projection onto the first coordinate, this subsumes these prior works as special cases, allowing for a general class of latent processes $(Z_t)_{0\leq t \leq 1}$ in a nearly arbitrary state space $\mathcal Z$, using the formalism of generator matching to allow for continuous, discrete, or manifold-valued processes.

In particular, the learnt process at \(t = 1\) samples from the distribution of \(\Phi(Y_1)\), which is the desired data distribution. We give sufficient conditions for a loss function to be valid in this general setting, recovering the results of the works cited above as corollaries. This result has broad applicability, enabling the construction of a wide array of new flow matching schemes by allowing for a more general class of latent spaces. As a concrete new application, we outline a non-projection \(\Phi \colon \mathcal{Y} \to \mathcal{X}\) with manifold-valued latents for protein structure generation that separates chain-level rigid-body motion from internal flexibility (\S\ref{sec:nonprojection}), where the particular chain-level versus residue-level or internal state is latent, and the model only sees the world state, which we plan to implement in future work.

\section{Earlier Work}\label{sec:earlierwork}

Several recent generative models train with the aid of a latent stochastic process that is marginalised out at generation time. We briefly survey these and note how the pushforward framework developed herein subsumes the special-case justifications given in each.

\subsection{Generator Matching}
Generator Matching \citep{holderrieth2025generatormatchinggenerativemodeling} unifies a large class of diffusion and flow matching models by showing that one may learn the marginal generator of a time-inhomogeneous Feller process by training against conditional generators, where conditioning is on a static latent random variable $Z$ (typically a pair of endpoints). The present work extends this to conditioning on a latent \textit{stochastic process}.

\subsection{Edit Flows, OneFlow and Flowception}
Edit Flows \citep{havasi2025editflowsflowmatching} solves variable-length discrete generation by adjoining a null token $\varepsilon$ to the alphabet and learning joint rates on an extended state space $\cX \times \cZ$. The projection onto the first coordinate (stripping the null tokens) induces rates on $\cX$ alone. Theorem~3.1 therein justifies this projection for CTMCs. The same technique is employed by OneFlow \citep{nguyen2025oneflowconcurrentmixedmodalinterleaved} but for interleaved token image generation and by Flowception \citep{ifriqi2025flowceptiontemporallyexpansiveflow} for video generation via learned frame-insertion rates. In the latter two, a continuous process depends non-trivially on a latent discrete process which must be marginalised out. The theorems in Edit Flows, however, are concerned only with the case where both the main and latent process are discrete. 

\subsection{Branching Flows}
Branching Flows \citep{billera2025branchingflowsdiscretecontinuous} introduces tree-structured conditional paths for variable-length generation on manifold, Euclidean and discrete state spaces. The branching structure creates an ambiguity over which element belongs to which branch, necessitating an auxiliary discrete process. The resulting \textit{auxiliary generator matching} result extends Theorem~3.1 of \citet{havasi2025editflowsflowmatching} to the case of an arbitrary ``main'' state space, while still requiring the latent process to be discrete.

\subsection{Any-Order Flexible Length Masked Diffusion}
Similarly to \citet{havasi2025editflowsflowmatching}, ``Any-order flexible length masked diffusion'' \citep{kim2025anyorderflexiblelengthmasked} solves the diffusion-model variable-length generation task. It uses flexible masked diffusion models, extending the continuous stochastic interpolants of \cite{albergo2025stochasticinterpolantsunifyingframework} to the discrete, variable-length setting by training a neural network to learn the unmasking posterior and an insertion expectation, which determine the insertion/unmasking rates. During training one augments with an auxiliary process \(s_t\) which keeps track of the indices of tokens as they are inserted. This auxiliary process is used during training and subsequently marginalised out for generation, similar to the branch-index tracking process of \citet{billera2025branchingflowsdiscretecontinuous}, except specialised for the case of a discrete state space. 

\section{Pushforward Generator Matching}
\subsection{Generator Matching}\label{subsec:gmrecap}

We briefly recall the Generator Matching framework of \cite{holderrieth2025generatormatchinggenerativemodeling}, which provides the theoretical foundation for the present work. Let $(\cX, d_\cX)$ be a Polish space. A time-inhomogeneous Feller process $(X_t)_{0 \leq t \leq 1}$ on $\cX$ is characterised by its infinitesimal generator $L_t$, which acts on test functions $f \in \cT(\cX) \subset C(\cX)$ and determines the evolution of the time-marginals $(p_t)_{0 \leq t \leq 1}$ via the Kolmogorov forward equation (KFE)
\begin{equation*}
    \partial_t \langle p_t, f \rangle = \langle p_t, L_t f \rangle, \qquad \forall f \in \cT(\cX),\quad t\in[0,1).
\end{equation*}
This framework encompasses a wide range of generative processes. For instance, when $\cX = \R^d$ and $L_t f(x) = b_t(x) \cdot \nabla f(x) + \tfrac{1}{2}\sigma_t^2(x) \Delta f(x)$, the KFE is the Fokker--Planck equation of a stochastic differential equation. When $\cX$ is a finite set and $L_t f(x) = \sum_{x'} u_t(x' \mid x)[f(x') - f(x)]$, it is the forward equation of a continuous-time Markov chain with rates $u_t$.

However, one does not need to learn the marginal generator $L_t$ directly. Instead, given a latent random variable $Z$, a conditional generator $L_t^z$ parametrised by $F_t^z(x)$ for each realisation $Z = z$ is defined and a neural network $F_t^\theta(x)$ is trained against these conditional targets. Training against the conditional loss recovers the correct marginal generator via a gradient equality \cite[Theorem~1]{holderrieth2025generatormatchinggenerativemodeling}.

\subsection{Linear Parametrisations of Infinitesimal Generators}\label{subsec:linparam}

The notion of a linear parametrisation, introduced in \cite{holderrieth2025generatormatchinggenerativemodeling} and extended to time- and state-dependent form in \cite{billera2025timedependentlossreweighting}, is central to the training objectives of Generator Matching.

\begin{definition}[Linear parametrisation of a generator]\label{def:linparam}
    Let $L_t \colon \cT(\cX) \to C(\cX)$ be the generator of a Feller process on a Polish space $\cX$. A \textit{(time- and state-dependent) linear parametrisation} of $L_t$ consists of:
    \begin{enumerate}[label=\textup{(\roman*)}]
        \item for each time $t$ and $x \in \cX$, an inner product space $(V_{t,x}, \langle \cdot, \cdot \rangle_{V_{t,x}})$;
        \item a linear map $\mathcal{K}_{t,x} \colon \cT(\cX) \to V_{t,x}$;
        \item a closed convex set $\Omega_{t,x} \subset V_{t,x}$ and a target function $F_t(x) \in \Omega_{t,x}$;
    \end{enumerate}
    such that for all $f \in \cT(\cX)$ it holds that
    \begin{equation}\label{eqn:linparam}
        L_t f(x) = \langle \mathcal{K}_{t,x} f, F_t(x) \rangle_{V_{t,x}}.
    \end{equation}
    The map $\mathcal{K}_{t,x}$ encodes the ``structural'' part of the generator (e.g.\ finite differences for a CTMC, or differential operators for a diffusion), while $F_t(x)$ captures the learnable parameters (e.g.\ rates or velocity fields). Training a neural network $F_t^\theta(x)$ to approximate $F_t(x)$ then recovers the generator $L_t$.
\end{definition}

\subsection{Marginal and Conditional Generators}
\begin{theorem}[Pushforward Kolmogorov Forward Equation]\label{thrm:pushforwardkfe}
    Let $(\cY,d_{\cY})$ be a Polish space and let $(Y_t)_{0\leq t\leq1}$ be a time-inhomogeneous Feller process on $\cY$ with infinitesimal generator $W_t$ and time-marginals $(p_t^{\cY})_{0\leq t\leq 1}$, satisfying the regularity conditions of \cite[Appendix~A.2]{holderrieth2025generatormatchinggenerativemodeling} (Assumption~\ref{ass:gm-cY}).
    Let $\Phi\colon\cY\to\cX$ be a measurable map into a measurable space $(\cX,\B(\cX))$ satisfying the domain compatibility conditions of Assumption~\ref{ass:Phi}, and let $p_t^{\cX}:=\Phi_{\#}p_t^{\cY}$. Assume the following integrability condition (Assumption~\ref{ass:integrability}): for every $f\in\cT(\cX)$ and every $t\in[0,1)$,
    \begin{equation}\label{eqn:integrability}
        \E_{y_t \sim p_t^{\cY}}[|W_t(f\circ\Phi)(y_t)|] < \infty.
    \end{equation}
    Define, for each $t\in[0,1)$ and each test function $f\in\cT(\cX)$,
    \begin{equation*}
        L_t f(x_t) := \E[W_t(f\circ\Phi)(Y_t)\mid\Phi(Y_t)=x_t].
    \end{equation*}
    Then $p_t^{\cX}$ satisfies the Kolmogorov forward equation with generator $L_t$:
    \begin{equation*}
        \partial_t \langle p_t^{\cX}, f\rangle = \langle p_t^{\cX}, L_t f\rangle, \qquad \forall f\in\cT(\cX),\quad t\in[0,1),
    \end{equation*}
    where $\langle \mu, g \rangle := \int g \, d\mu$. If additionally the KFE with generator $L_t$ uniquely determines $(p_t^{\cX})_{0 \leq t \leq 1}$ (Assumption~\ref{ass:gm-cX}), then $L_t$ may be used within the Generator Matching framework of \cite{holderrieth2025generatormatchinggenerativemodeling}. The proof is given in Appendix~\ref{subsec:proof}; all assumptions are stated formally in Appendix~\ref{subsec:hypotheses}.
\end{theorem}
\begin{remark}
    Assumption~\ref{ass:gm-cX} is non-trivial and a regularity-violating example is given in \S\ref{subs:regularity-violating}.
\end{remark}
For many cases we will be concerned with ``marginalising'' a latent process to learn a base process. We therefore state this special case as a corollary.
\begin{corollary}[Generator of a projected process]\label{cor:projectedprocess}
    Specialising Theorem~\ref{thrm:pushforwardkfe} to $\cY = \cX \times \cZ$ and $\Phi = \pi_{\cX}$, the projection onto the first coordinate: let $(X_t, Z_t)$ be a Feller process on $\cX \times \cZ$ with infinitesimal generator $W_t$, satisfying the hypotheses of Theorem~\ref{thrm:pushforwardkfe}. Then
    \begin{equation*}
        L_tf(x_t) := \E[W_t(f \circ \pi_{\cX})(X_t,Z_t)\mid X_t=x_t]
    \end{equation*}
    generates a process on $\cX$ with marginal measure $p_t^{\cX} = [\pi_{\cX}]_{\#} p_t^{\cX \times \cZ}$, that is, $L_t$ solves the Kolmogorov Forward Equation
    \begin{equation*}
        \partial_t \langle p_t^{\cX}, f\rangle = \langle p_t^{\cX}, L_t f\rangle, \qquad t\in[0,1).
        \end{equation*}
\end{corollary}
\begin{remark}[Compact-$\cZ$.]\label{rmk:compactZ}
    When $\cZ$ is compact, integrability along the $\cZ$-fibre is automatic by boundedness, and the domain condition reduces to the $\cX$-factor. For flow matching and diffusion, it suffices that $f \in C_0^k(\cX)$ for $k=1$ and $k=2$ respectively, and for CTMCs it is automatic. This covers any finite-state latent $Z_t$ (\S\ref{subs:subsumeef}, \S\ref{subs:1d-example}) and flow matching and diffusion on compact Riemannian manifolds such as $\SO(3)$.
\end{remark}
\begin{remark}\label{rmk:subsumeef}
    In \cite{havasi2025editflowsflowmatching} it is proven that if \(u_t(x',z'\mid x_t,z_t)\) is a rate on \(\mathcal{X} \times \mathcal Z\) that generates \(p_t(x,z)\) then 
    \begin{equation*}
        u_t(x'\mid x_t) := \sum_{z' \in \mathcal{Z}} \E_{z_t \sim p_t(\cdot \mid x_t)}[u_t(x',z'\mid x_t,z_t)]
    \end{equation*}
    generates an \(x\)-marginal path \(p_t^{\cX}(x)\). This result follows as a special case of Corollary~\ref{cor:projectedprocess}. For a detailed discussion, refer to \S\ref{subs:subsumeef}.
\end{remark}

\subsection{Latent Process Conditional Generator Matching Loss}\label{subsec:pushforwardcgm}
In this section we establish which loss functions are valid for training a neural network to learn the pushforward generator \(L_t\) of Theorem~\ref{thrm:pushforwardkfe}. One may train against a \textit{conditional} target which is defined pointwise for each realisation \(Y_t = y\) of the latent process and still recover the correct \textit{marginal} generator via an equality of gradients. This generalises the conditional generator matching paradigm of \cite{holderrieth2025generatormatchinggenerativemodeling} from conditioning on a static latent variable to conditioning on a latent stochastic process, and subsumes the special cases proven in \cite{havasi2025editflowsflowmatching, billera2025branchingflowsdiscretecontinuous}.

\begin{definition}\label{def:condgen}
    (Latent process conditional generator). In the setting of Theorem~\ref{thrm:pushforwardkfe}, for each \(y_t \in \mathcal{Y}\) define the \textit{conditional generator given \(Y_t = y_t\)} at the point \(x_t := \Phi(y_t)\) by
    \begin{equation*}
        L_t^{y_t} f(x_t) := W_t(f \circ \Phi)(y_t), \qquad f \in \mathcal{T}(\mathcal{X}).
    \end{equation*}
    This defines \(L_t^{y_t}f\) only at the single point \(x_t = \Phi(y_t)\), not as a function on all of~\(\cX\). When two distinct \(y, y' \in \cY\) satisfy \(\Phi(y) = \Phi(y')\), the conditional generators \(L_t^y\) and \(L_t^{y'}\) may assign different values at the same point. This is expected, as they correspond to different conditioning events (since $W_t$ may introduce $y$-dependence beyond $\Phi(y)$). By the definition of the pushforward generator in Theorem~\ref{thrm:pushforwardkfe}, the marginal and conditional generators are related by 
    \begin{equation}\label{eqn:margcondrelation}
        L_t f(x_t) = \E[L_t^{Y_t}f(x_t) \mid \Phi(Y_t) = x_t].
    \end{equation}
\end{definition}
\begin{definition}\label{def:condlinparam}
    (Conditional linear parametrisation). Suppose the conditional generator \(L_t^{y_t}\) of Definition~\ref{def:condgen} admits a time- and state-dependent linear parametrisation
    \begin{equation*}
        L_t^{y_t}f(\Phi(y_t)) = \inner{\mathcal{K}_{t,\Phi(y_t)}f}{F_t^{y_t}(\Phi(y_t))}_{V_{t,\Phi(y_t)}}
    \end{equation*}
    where \(\mathcal{K}_{t,x_t}\) and \(V_{t,x_t}\) are as in the linear parametrisation of \(L_t\) and \(F_t^{y_t}(\Phi(y_t)) \in \Omega_{t,\Phi(y_t)}\). Then by \eqref{eqn:margcondrelation} and linearity of the inner product, the marginal generator inherits a linear parametrisation with
    \begin{equation}\label{eqn:margfromcond}
        F_t(x_t) = \E[F_t^{Y_t}(x_t) \mid \Phi(Y_t) = x_t].
    \end{equation}
\end{definition}
\begin{remark}
    Observe that in the above, the linear parametrisation may only depend on the value $x_t := \Phi(y_t)$ for each $y_t\in \mathcal Y$.
\end{remark}
\begin{definition}\label{def:condgmloss}
    (Pushforward Conditional Generator Matching Loss). In the setting of Definition~\ref{def:condlinparam}, let \(D_{t,x_t} : \Omega_{t,x_t} \times \Omega_{t,x_t} \to \R\) be a Bregman divergence and let \(F_t^\theta\) be a neural network with \(F_t^\theta(x_t) \in \Omega_{t,x_t}\) for all \(t, x_t\). Assume that for all $t\in [0,1)$,\[\E_{y_t \sim p_t^{\cY}}[D_{t,\Phi(y_t)}(F_t^{y_t}(\Phi(y_t)),F_t^\theta(\Phi(y_t)))] < \infty.\] 
    The latent process conditional generator matching loss is 
    \begin{equation*}
        L_{\mathrm{cgm}}(\theta) := \E_{t \sim U[0,1], y_t \sim p_t^{\mathcal{Y}}(dy_t)}[D_{t,\Phi(y_t)}(F_t^{y_t}(\Phi(y_t)),F_t^\theta(\Phi(y_t)))].
    \end{equation*}
    Equivalently, by disintegrating \(p_t^{\mathcal{Y}}\) along \(\Phi\) (Corollary~\ref{cor:conditional-law}),
    \begin{equation*}
        L_{\mathrm{cgm}}(\theta) = \E_{t \sim U[0,1], x_t \sim p_t^{\mathcal{X}}, y_t \sim p_t^{\mathcal{Y}}(dy_t \mid \Phi(Y_t)=x_t)}[D_{t,x_t}(F_t^{y_t}(x_t),F_t^\theta(x_t))].
    \end{equation*}
\end{definition}
\begin{definition}\label{def:marggmloss}
    (Marginal Generator Matching Loss). In the setting of Definition~\ref{def:condlinparam}, the \textit{marginal generator matching loss} is 
    \begin{equation*}
        L_{\mathrm{gm}}(\theta) := \E_{t \sim U[0,1], x_t \sim p_t^\mathcal{X}}[D_{t,x_t}(F_t(x_t),F_t^\theta(x_t))].
    \end{equation*}
\end{definition}
\begin{theorem}\label{thrm:losseshavesamegrads}
    (Gradient equality of Conditional and Marginal Generator Matching Losses). Let \(L_\mathrm{cgm}(\theta), L_\mathrm{gm}(\theta)\) be as in Definitions~\ref{def:condgmloss} and~\ref{def:marggmloss}, and suppose the integrability conditions of Definition~\ref{def:condgmloss} hold. Then 
    \begin{equation*}
        \nabla_\theta L_\mathrm{cgm}(\theta) = \nabla_\theta L_\mathrm{gm}(\theta). 
    \end{equation*}
    In particular, training against the conditional loss \(L_\mathrm{cgm}\) (which only requires samples from the process \(Y_t\)) recovers the same stationary point as training against the marginal loss \(L_\mathrm{gm}\), which depends on the generally intractable marginal target \(F_t\).
\end{theorem}
\begin{proof}
    See Appendix~\ref{sec:gradientproof}.
\end{proof}

\subsection{Conditioning on a latent process \(Z_t\)}\label{subsec:latentprocess}
In many applications, the process of interest \(X_t\) evolves jointly with an auxiliary \textit{latent} process \(Z_t\) that is used during training but discarded at generation time. This corresponds to the product-space specialisation of the pushforward framework: we set \(\mathcal{Y} = \mathcal{X} \times \mathcal{Z}\), take \(\Phi = \pi_{\mathcal{X}}\) to be the canonical projection, and write \(Y_t = (X_t, Z_t)\). The pushforward marginals are then \(p_t^{\mathcal{X}} = (\pi_{\mathcal{X}})_\# p_t^{\mathcal{Y}}\), which is simply the \(X\)-marginal of the joint law, and the conditional law \(p_t^{\mathcal{Y}}(dy_t \mid x_t)\) from Corollary~\ref{cor:conditional-law} reduces to the conditional law of \(Z_t\) given \(X_t = x_t\), which we denote \(p_t(dz \mid x_t)\).

In this setting, the definitions of \S\ref{subsec:pushforwardcgm} take the following concrete form. The conditional generator (Definition~\ref{def:condgen}) for \(y_t = (x_t,z_t) \in \mathcal{X} \times \mathcal{Z}\) is
\begin{equation*}
    L_t^{y_t}f(x_t) = W_t(f \circ \pi_{\mathcal{X}})(x_t,z_t).
\end{equation*}
\begin{remark}
    Since \(\Phi = \pi_{\mathcal{X}}\), the evaluation point \(\Phi(y_t) = x_t\) is already determined by the projection, so the dependence of the conditional generator on the \(x_t\)-component of \(y_t = (x_t,z_t)\) is vacuous: the right-hand side \(W_t(f \circ \pi_{\mathcal{X}})(x_t,z_t)\) depends on \(x_t\) only through the test function, not through the conditioning. Accordingly, we write \(L_t^{z_t}\) in place of \(L_t^{y_t}\):
    \begin{equation*}
        L_t^{z_t}f(x_t) := W_t(f \circ \pi_{\mathcal{X}})(x_t,z_t).
    \end{equation*}
\end{remark}
The marginal generator (Corollary~\ref{cor:projectedprocess}) is
\begin{equation*}
    L_tf(x_t) = \E[L_t^{Z_t}f(x_t) \mid X_t = x_t] = \int_{\mathcal{Z}} W_t(f \circ \pi_{\mathcal{X}})(x_t,z) p_t(dz \mid x_t).
\end{equation*}
If the conditional generator admits a linear parametrisation with \(F_t^{z_t}(x_t)\) (Definition~\ref{def:condlinparam}), the marginal target is
\begin{equation*}
    F_t(x_t) = \E[F_t^{Z_t}(x_t) \mid X_t = x_t].
\end{equation*}
The conditional generator matching loss (Definition~\ref{def:condgmloss}) becomes
\begin{equation*}
    L_{\mathrm{cgm}}(\theta) = \E_{t \sim U[0,1], (x_t,z_t) \sim p_t^{\mathcal{X} \times \mathcal{Z}}}[D_{t,x_t}(F_t^{z_t}(x_t), F_t^\theta(x_t))],
\end{equation*}
and Theorem~\ref{thrm:losseshavesamegrads} guarantees \(\nabla_\theta L_{\mathrm{cgm}}(\theta) = \nabla_\theta L_{\mathrm{gm}}(\theta)\). Thus, one may sample jointly from \((X_t, Z_t)\), use the latent state \(Z_t\) to compute the conditional training target \(F_t^{Z_t}\), and train a network \(F_t^\theta\) that receives only \(X_t\) as input. Equivalently, 
\begin{equation*}
    L_{\mathrm{cgm}}(\theta) = \E_{t \sim U[0,1], x_t \sim p_t^{\cX}(dx), z_t \sim p_t^{\mathcal Z}(dz\mid X_t=x_t)}[D_{t,x_t}(F_t^{z_t}(x_t), F_t^\theta(x_t))],
\end{equation*}

\begin{remark}[Fixed-endpoint conditioning with a latent process]\label{rmk:fixedendpoint}
    The conditioning on \(Z_t\) is compatible with additionally conditioning on a \emph{static} latent variable \(Z\) (e.g.\ a pair of endpoints) in the sense of \cite{holderrieth2025generatormatchinggenerativemodeling}. This double conditioning perspective, developed in \S\ref{subsec:doubleconditioning}, provides a natural sampling procedure for the loss.
\end{remark}

\subsection{Fixed-endpoint conditioning with a latent process}\label{subsec:doubleconditioning}

The conditional generator matching loss of Definition~\ref{def:condgmloss} requires sampling from the joint law \(p_t^{\mathcal{Y}}(dy_t)\). In many practical settings, this joint law is itself constructed by first choosing a static latent variable \(z\) (typically an endpoint or pair of endpoints) and then running the process conditionally on \(z\). Indeed, suppose that 
\begin{equation*}
    p_t^{\mathcal{Y}}(dy_t) = \int p^{\mathcal{Y}}(dy_t \mid z) p_{\mathcal Z}(dz),
\end{equation*}
where \(z \in \mathcal{Z}\) is a static random variable (e.g., points used to steer the conditional path) with law \(p_{\mathcal Z}\), and \(p^{\mathcal{Y}}(dy_t \mid z)\) is the conditional law of the process at time \(t\) given \(z\). For each realisation \(z\) and each \(y_t\) in the support of \(p^{\mathcal{Y}}(\cdot \mid z)\), the conditional generator
\begin{equation*}
    L_t^{z, y_t} f(\Phi(y_t)) := W_t^z(f \circ \Phi)(y_t)
\end{equation*}
depends both on the fixed latent state \(z\) and the current process state \(y_t\). This may be viewed from two complementary angles:
\begin{enumerate}[label=\textup{(\roman*)}]
    \item \textit{As a latent-process problem (\S\ref{subsec:latentprocess}):} marginalise over \(z\) to obtain the law \(p_t^{\mathcal{Y}}(dy_t)\). In this case, the conditional generator is \(L_t^{y_t}\) and the framework of \S\ref{subsec:pushforwardcgm} applies directly.
    \item \textit{As a fixed-endpoint problem with process-valued conditionals:} for each \(z\), marginalise over \(y_t \mid z\) to obtain a generator \(L_t^z f(x_t) := \E_{y_t \sim p^{\mathcal{Y}}(\cdot \mid z)}[L_t^{z,y_t}f(x_t)]\) conditioned only on the static latent, recovering the setting of \cite{holderrieth2025generatormatchinggenerativemodeling}. Marginalising further over \(z\) recovers the full marginal generator \(L_t\).
\end{enumerate}
Both viewpoints lead to the same marginal generator and the same gradient equality (Theorem~\ref{thrm:losseshavesamegrads}), but the double-conditioning perspective makes explicit a practical sampling procedure for the conditional loss. Substituting the disintegration into the loss of Definition~\ref{def:condgmloss} gives
\begin{equation}\label{eqn:doublecondsamplingprocedure}
    L_{\mathrm{cgm}}(\theta) = \E_{t \sim U[0,1], z \sim p_{\mathcal Z}(dz), y_t \sim p^{\mathcal{Y}}(dy_t \mid z)}[D_{t,\Phi(y_t)}(F_t^{z,y_t}(\Phi(y_t)), F_t^\theta(\Phi(y_t)))],
\end{equation}
which suggests the following training procedure: \textbf{(1)} sample fixed latents \(z \sim p_{\mathcal Z}\); \textbf{(2)} given \(z\), sample the latent process state \(y_t \sim p^{\mathcal{Y}}(dy_t \mid z)\); \textbf{(3)} compute the conditional target \(F_t^{z,y_t}\) and evaluate the Bregman divergence.

\subsubsection{Edit Flows}\label{subs:subsumeef}
We now show how the theory of Edit Flows \citep{havasi2025editflowsflowmatching} is recovered as a special case of the pushforward conditional generator matching framework of \S\ref{subsec:pushforwardcgm}--\S\ref{subsec:latentprocess}. Let \(\cX\) and \(\cZ\) be finite sets and let \((X_t,Z_t)\) be a CTMC on the joint state space \(\cX \times \cZ\) with rates \(u_t(x',z' \mid x_t, z_t) \geq 0\). Applying the joint generator $W_t$ of $(X_t,Z_t)$ to a test function of the form \(f \circ \pi_{\cX}\) and using Corollary~\ref{cor:projectedprocess}, one can show that
\begin{equation}\label{eqn:efprojectedrates}
    u_t(x' \mid x_t) = \sum_{z'}\E_{z_t \sim p_t(\cdot \mid x_t)}[u_t(x',z' \mid x_t, z_t)] \qquad \text{generates} \qquad p_t^{\cX}(x) = \sum_z p_t(x,z),
\end{equation}
recovering the first part of Theorem~3.1 of \citep{havasi2025editflowsflowmatching}. Moreover, the conditional generator admits a linear parametrisation (Definition~\ref{def:condlinparam}) in which the \(\cZ\)-dependence is confined entirely to the target vector \(F_t^{z_t}(x_t) = (\tilde{u}_t(x' \mid x_t, z_t))_{x' \in \cX}\), where \(\tilde{u}_t(x' \mid x_t, z_t) := \sum_{z'} u_t(x',z' \mid x_t, z_t)\) is the total rate of the \(\cX\)-component jumping to \(x'\). The gradient equality of the conditional and marginal losses (the second part of Theorem~3.1 in \citep{havasi2025editflowsflowmatching}) then follows directly from Theorem~\ref{thrm:losseshavesamegrads}. The complete derivations are given in Appendix~\ref{app:editflows}.

\subsection{Examples}\label{sec:examples}

\subsubsection{A one-dimensional example}\label{subs:1d-example}

We illustrate the latent-process framework of \S\ref{subsec:latentprocess} with a regime-switching diffusion whose conditional paths exhibit discontinuous changes of drift direction, yet whose marginal generator is a standard diffusion generator that can be learnt by a network receiving only the continuous component \(X_t\).
\begin{figure}
    \centering
    \includegraphics[width=1\linewidth]{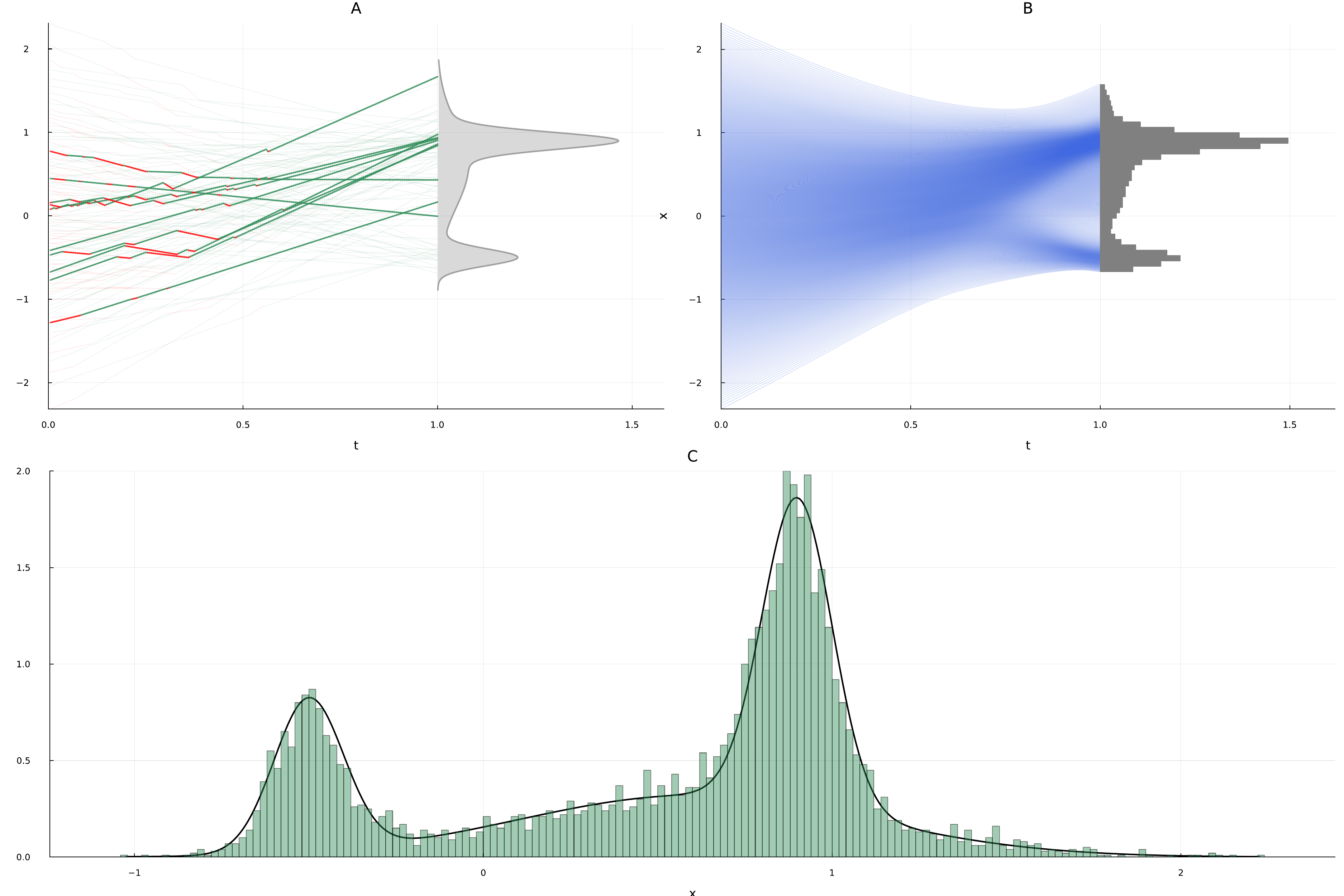}
    \caption{A) Conditional trajectories (training target) with switching, colored by the state of the latent process, where 10 conditional sample paths are foregrounded, B) Model-learnt marginal trajectories, C) Marginal distribution at \(t=1\), with histogram of generated samples matching the target distribution.}
    \label{fig:switchingdriftexample}
\end{figure}
Consider a joint process \((X_t,Z_t)\) on \(\R \times \{-1,+1\}\), with fixed endpoints \(x_0, x_1 \in \R\), evolving according to the following dynamics:
\begin{align*}
    &Z_t \mid (X_t=x_t, X_0=x_0, X_1=x_1) \sim \begin{pmatrix}
        -\lambda_1(x_t,t) & \lambda_1(x_t,t) \\  \lambda_2(x_t,t) & -\lambda_2(x_t,t)
    \end{pmatrix};\\[6pt]
    & X_t \mid (Z_t = z_t, X_0=x_0, X_1=x_1) \sim \frac{z_t x_1 - X_t}{1-t}dt + \sigma_t dB_t.
\end{align*}
where \(B_t\) is a standard Brownian motion. That is, conditional on the continuous component \(X_t\) and the endpoints, \(Z_t\) evolves as a CTMC with rates \(\lambda_1(x_t,t), \lambda_2(x_t,t)\). Conditional on the CTMC value and the endpoints, \(X_t\) evolves as a Brownian bridge towards \(z_tx_1\), that is, towards the target \(x_1\) when \(Z_t = +1\) and towards \(-x_1\) when \(Z_t = -1\). The rates \(\lambda_1,\lambda_2\) are chosen so that \(X_1 = x_1\) a.s.\ in the conditional paths (see Appendix~\ref{app:mixedconddiscrete} for the construction). See Figure~\ref{fig:switchingdriftexample} A for representative trajectories.

Since \(Z_t\) and \(X_t\) evolve in a dependent manner, this falls outside the scope of \cite{holderrieth2025generatormatchinggenerativemodeling}, which treats conditioning on a static latent variable \(Z\). By Corollary~\ref{cor:projectedprocess}, the marginal generator on \(\R\) is nonetheless well-defined, and by Theorem~\ref{thrm:losseshavesamegrads} a network that receives only \(X_t\) may be trained against the conditional loss. This is an instance of the double conditioning of \S\ref{subsec:doubleconditioning}: the static latent is the pair of endpoints \((x_0,x_1)\) and the process-valued latent is \(z_t\). In Appendix~\ref{app:mixedconddiscrete} we derive the conditional generator and construct an \(x_1\)-prediction linear parametrisation (Definition~\ref{def:condlinparam}). Taking the Bregman divergence to be the squared Euclidean norm, the training loss reduces to
\begin{equation*}
    \E_{t \sim U[0,1], (x_0,x_1) \sim q, (x_t, z_t) \sim p_t(\cdot \mid x_0,x_1)}
    [\lVert z_t x_1 - x_1^\theta(x_t)
    \rVert^2],
\end{equation*}
where the latent state \(z_t\) appears in the target but is not passed to the network. At generation time, the velocity is recovered via \(u_t^\theta(x) = (x_1^\theta(x) - x)/(1-t)\), with no reference to the latent dynamics. Figure~\ref{fig:switchingdriftexample} confirms that the learnt marginal paths are smooth, and the switching dynamics of the conditional training target are effectively integrated out, while still retaining the property that the marginal distribution at \(t = 1\) matches the target.

\subsubsection{A regularity-violating example}\label{subs:regularity-violating}

One may ask whether regularity of the base process \(Y_t\) and smoothness of the map \(\Phi\) suffice to guarantee that the pushforward process \(\Phi(Y_t)\) satisfies the KFE sufficiency condition (Assumption~\ref{ass:gm-cX}). The following example shows that this need not be the case, even for \(\Phi \in C^\infty\).

\begin{example}\label{ex:regularity-violating}
    Define \(\Phi \colon \R \to \R_{\geq 0}\) by
    \begin{equation*}
        \Phi(x) := \begin{cases}
            e^{-1/x} & x > 0, \\
            0         & \text{otherwise.}
        \end{cases}
    \end{equation*}
    
    This function belongs to \(C^\infty(\R)\). Let \((B_t)_{t \geq 0}\) be a standard Brownian motion, so that \((B_t, \tfrac{1}{2}\partial_{xx})\) satisfies the Generator Matching regularity conditions of \cite{holderrieth2025generatormatchinggenerativemodeling}. The pushforward \(\Phi(B_t)\) has an atom at \(0\) for every \(t > 0\), since \(\PR[\Phi(B_t) = 0] = \PR[B_t \leq 0] = \tfrac{1}{2}\).

    The corresponding marginal generator $L_t$ satisfies \(L_t f(0) = 0\) for every test function~\(f\). In Appendix~\ref{app:regularity-violating}, we describe a one-parameter family of probability paths satisfying the KFE, violating the uniqueness assumption on the KFE. 
\end{example}

\section{Non-Projection Pushforwards}\label{sec:nonprojection}

All examples in \S\ref{sec:examples} and in the prior work discussed in \S\ref{sec:earlierwork} take $\Phi = \pi_\cX$ to be a canonical projection. \footnote{Any pushforward $\Phi\colon \cY \to \cX$ can formally be subsumed by the projection framework: one embeds $Y_t \mapsto (\Phi(Y_t), Y_t) \in \cX \times \cY$ and projects onto the first coordinate. However, this reduction is artificial. The ``latent'' component becomes the entire original process $Y_t$, the joint law on $\cX \times \cY$ is degenerate, and the geometric structure of $\Phi$ is obscured.}

Protein complexes are formed by multiple `chains`, where residues within a chain are covalently linked. We might wish to train a model via a conditional path that respects this structure. Each residue in a protein backbone can be modelled by a point in $\R^3$ and an orientation in $\SO(3)$ \citep{jumper2021highly}. Further, let each chain have a shared orientation $R_{\mathrm{ch},t} \in \SO(3)$, a centroid $\mu_t \in \R^3$, per-residue internal rotations $R_{\mathrm{res},t}^{(i)} \in \SO(3)$, and per-residue centroid-relative displacements $p_t^{(i)} \in \R^3$. For each chain, the latent space is
\begin{equation*}
    \cY = \SO(3) \times \R^3 \times \SO(3)^n \times (\R^3)^n,
\end{equation*}
and the map to observed (world-frame) coordinates is $\Phi \colon \cY \to \SO(3)^n \times (\R^3)^n$,
\begin{equation*}
    \Phi(R_{\mathrm{ch},t}, \mu_t, (R_{\mathrm{res},t}^{(i)})_i, (p_t^{(i)})_i) = (R_{\mathrm{ch},t} R_{\mathrm{res},t}^{(i)}, R_{\mathrm{ch},t} p_t^{(i)} + \mu_t)_i.
\end{equation*}
The chain rotation $R_{\mathrm{ch},t}$ acts on the internal geometry (orientations and displacements) while leaving the centroid $\mu_t$ invariant. When $Y_t$ follows an SDE on $\cY$, the pushforward through $\Phi$ induces correlated rotational and translational noise at the observed level: all residue orientations are simultaneously rotated, and all positions are coherently tumbled about the centroid. This captures chain-level rigid-body motion, while the per-residue components $R_{\mathrm{res},t}^{(i)}$ and $p_t^{(i)}$ handle internal flexibility. 

\section{Discussion}

Here we have expanded Generator Matching to allow for conditional processes with latent time-dependent variables with rich state spaces and dynamics. Recent examples have demonstrated the utility of discrete-valued time-dependent latent variables, and we anticipate the value of an expanded latent state space. In future work we plan to attempt to relax the regularity conditions for this, deepen our understanding of when they might be violated in real practical examples, and to develop and implement chain-level rotation example from section \ref{sec:nonprojection}.

\section{Acknowledgements}
This project received support from the Swedish Research Council (2024-00390 and 2023-02516) and the Knut and Alice Wallenberg Foundation (2024.0039) to B.M.

\bibliographystyle{plainnat}
\bibliography{iclr2026_delta}

@misc{holderrieth2025generatormatchinggenerativemodeling,
      title={Generator Matching: Generative modeling with arbitrary Markov processes}, 
      author={Peter Holderrieth and Marton Havasi and Jason Yim and Neta Shaul and Itai Gat and Tommi Jaakkola and Brian Karrer and Ricky T. Q. Chen and Yaron Lipman},
      year={2025},
      eprint={2410.20587},
      archivePrefix={arXiv},
      primaryClass={cs.LG},
      url={https://arxiv.org/abs/2410.20587}, 
}

@misc{billera2025timedependentlossreweighting,
      title={Time dependent loss reweighting for flow matching and diffusion models is theoretically justified}, 
      author={Lukas Billera and Hedwig Nora Nordlinder and Ben Murrell},
      year={2025},
      eprint={2511.16599},
      archivePrefix={arXiv},
      primaryClass={stat.ML},
      url={https://arxiv.org/abs/2511.16599}, 
}

@misc{lipman2024flowmatchingguidecode,
      title={Flow Matching Guide and Code}, 
      author={Yaron Lipman and Marton Havasi and Peter Holderrieth and Neta Shaul and Matt Le and Brian Karrer and Ricky T. Q. Chen and David Lopez-Paz and Heli Ben-Hamu and Itai Gat},
      year={2024},
      eprint={2412.06264},
      archivePrefix={arXiv},
      primaryClass={cs.LG},
      url={https://arxiv.org/abs/2412.06264}, 
}

@book{alma991018313432607596,
author = {Bogachev, Vladimir I.},
address = {Berlin, Heidelberg},
booktitle = {Measure Theory},
edition = {1st ed. 2007.},
isbn = {1-280-74570-3},
language = {eng},
publisher = {Springer Berlin Heidelberg},
series = {Measure Theory ; 1},
title = {Measure Theory },
year = {2007},
}

@misc{havasi2025editflowsflowmatching,
      title={Edit Flows: Flow Matching with Edit Operations}, 
      author={Marton Havasi and Brian Karrer and Itai Gat and Ricky T. Q. Chen},
      year={2025},
      eprint={2506.09018},
      archivePrefix={arXiv},
      primaryClass={cs.LG},
      url={https://arxiv.org/abs/2506.09018}, 
}

@misc{billera2025branchingflowsdiscretecontinuous,
      title={Branching Flows: Discrete, Continuous, and Manifold Flow Matching with Splits and Deletions}, 
      author={Lukas Billera and Hedwig Nora Nordlinder and Jack Collier Ryder and Anton Oresten and Aron Stålmarck and Theodor Mosetti Björk and Ben Murrell},
      year={2025},
      eprint={2511.09465},
      archivePrefix={arXiv},
      primaryClass={stat.ML} 
}

@misc{nguyen2025oneflowconcurrentmixedmodalinterleaved,
      title={OneFlow: Concurrent Mixed-Modal and Interleaved Generation with Edit Flows}, 
      author={John Nguyen and Marton Havasi and Tariq Berrada and Luke Zettlemoyer and Ricky T. Q. Chen},
      year={2025},
      eprint={2510.03506},
      archivePrefix={arXiv},
      primaryClass={cs.AI} 
}

@misc{ifriqi2025flowceptiontemporallyexpansiveflow,
      title={Flowception: Temporally Expansive Flow Matching for Video Generation}, 
      author={Tariq Berrada Ifriqi and John Nguyen and Karteek Alahari and Jakob Verbeek and Ricky T. Q. Chen},
      year={2025},
      eprint={2512.11438},
      archivePrefix={arXiv},
      primaryClass={cs.CV}
}

@misc{kim2025anyorderflexiblelengthmasked,
      title={Any-Order Flexible Length Masked Diffusion}, 
      author={Jaeyeon Kim and Lee Cheuk-Kit and Carles Domingo-Enrich and Yilun Du and Sham Kakade and Timothy Ngotiaoco and Sitan Chen and Michael Albergo},
      year={2025},
      eprint={2509.01025},
      archivePrefix={arXiv},
      primaryClass={cs.LG},
      url={https://arxiv.org/abs/2509.01025}, 
}

@misc{albergo2025stochasticinterpolantsunifyingframework,
      title={Stochastic Interpolants: A Unifying Framework for Flows and Diffusions}, 
      author={Michael S. Albergo and Nicholas M. Boffi and Eric Vanden-Eijnden},
      year={2025},
      eprint={2303.08797},
      archivePrefix={arXiv},
      primaryClass={cs.LG},
      url={https://arxiv.org/abs/2303.08797}, 
}

@article{jumper2021highly,
	title = {Highly accurate protein structure prediction with {AlphaFold}},
	volume = {596},
	issn = {},
	url = {},
	doi = {10.1038/s41586-021-03819-2},
	language = {en},
	number = {7873},
	urldate = {2026-01-28},
	journal = {Nature},
	author = {Jumper, John and Evans, Richard and Pritzel, Alexander and Green, Tim and Figurnov, Michael and Ronneberger, Olaf and Tunyasuvunakool, Kathryn and Bates, Russ and Žídek, Augustin and Potapenko, Anna and Bridgland, Alex and Meyer, Clemens and Kohl, Simon A. A. and Ballard, Andrew J. and Cowie, Andrew and Romera-Paredes, Bernardino and Nikolov, Stanislav and Jain, Rishub and Adler, Jonas and Back, Trevor and Petersen, Stig and Reiman, David and Clancy, Ellen and Zielinski, Michal and Steinegger, Martin and Pacholska, Michalina and Berghammer, Tamas and Bodenstein, Sebastian and Silver, David and Vinyals, Oriol and Senior, Andrew W. and Kavukcuoglu, Koray and Kohli, Pushmeet and Hassabis, Demis},
	year = {2021},
	pages = {583--589},
}

\newpage

\appendix

\section{Proof of the Pushforward Kolmogorov Forward Equation}\label{sec:appendix}

In this appendix we collect the measure-theoretic background needed to rigorously define the marginal generator $L_t f(x_t):=\E[W_t(f\circ\Phi)(Y_t)\mid\Phi(Y_t)=x_t]$ and to prove that it governs the Kolmogorov forward equation for the pushforward measure $p_t^{\cX}=\Phi_\# p_t^{\cY}$. All results in \S\ref{subsec:condexp}--\S\ref{subsec:pushforward} are standard and are cited from \cite{alma991018313432607596}. We reproduce precise statements here for completeness.

\begin{remark}[Notation]\label{rmk:notation}
Our notation differs from \cite{alma991018313432607596} in two respects. \paragraph{Conditional expectation.} We use the standard notation $\E[f\mid\B]$ and $\E[f\mid\eta]$ for the conditional expectation of $f$ with respect to a sub-$\sigma$-algebra~$\B$ or the $\sigma$-algebra generated by a random variable~$\eta$, respectively. In~\cite{alma991018313432607596} the same object is written $\mathbb{E}^{\B}f$, or $\mathbb{E}^{\B}_\mu f$.

\paragraph{Markov kernels.} In~\cite{alma991018313432607596}, a transition probability is written $P(\cdot\mid\cdot)\colon X_1\times\B_2\to\R$ with the convention that $P(x\mid B)$ places the conditioning point $x$ in the first slot and the measurable set $B$ in the second. We instead write $\kappa(B\mid x)$, which emphasises the conditional-probability interpretation: $\kappa(\cdot\mid x)$ is a probability measure on~$\B_2$ for each fixed~$x$, while $\kappa(B\mid\cdot)$ is a $\B_1$-measurable function for each fixed~$B$. Integrals against the kernel are written $\int f(y)\kappa(dy\mid x)$.
\end{remark}


\subsection{Conditional expectation: definition and existence}\label{subsec:condexp}
\begin{definition}
  \citep[Def. 2.12.2]{alma991018313432607596} Let \(\mathcal F\) be some collection of functions in \(X\) that have the same codomain. The smallest \(\sigma\)-algebra on \(X\) in which all functions belonging to \(\mathcal F\) are measurable is called the \(\sigma\)-algebra generated by \(\mathcal F\) and is denoted by \(\sigma(\mathcal F)\) (formally, this can be defined as the intersection of all $\sigma$-algebras for which each function in $\mathcal F$ is measurable). In particular, the \(\sigma\)-algebra generated by a random variable \(\eta\) is the smallest \(\sigma\)-algebra in which \(\eta\) is measurable and is denoted by \(\sigma(\eta)\).
\end{definition}

\begin{definition}[Conditional expectation {\cite[Def.~10.1.1]{alma991018313432607596}}]
\label{def:condexp}
Let $(\Omega,\A,\mu)$ be a probability space and let $\B\subset\A$ be a sub-$\sigma$-algebra. For $f\in L^1(\mu)$, a \emph{conditional expectation of $f$ given $\B$} is a $\B$-measurable, $\mu$-integrable function $\E[f\mid\B]$ satisfying
\begin{equation}\label{eq:condexp-def}
\int_\Omega g fd\mu = \int_\Omega g\E[f\mid\B]d\mu
\end{equation}
for every bounded $\B$-measurable function $g$. Equivalently,
\begin{equation}\label{eq:condexp-sets}
\int_B fd\mu = \int_B \E[f\mid\B]d\mu, \qquad \forall B\in\B.
\end{equation}
When $\B=\sigma(\eta)$ is generated by a measurable mapping (i.e.\ random variable) $\eta$, we write $\E[f\mid\eta]$ in place of $\E[f\mid\sigma(\eta)]$.
\end{definition}
\begin{theorem}[Existence and basic properties {\cite[Thm.~10.1.5 (1)-(2)]{alma991018313432607596}}]
\label{thm:condexp-exist}
Let $\mu$ be a probability measure on $(\Omega,\A)$ and let $\B\subset\A$ be a sub-$\sigma$-algebra. For every $f\in L^1(\mu)$ there exists a $\B$-measurable function $\E[f\mid\B]$, unique $\mu$-a.e., such that:
\begin{enumerate}
 \item[\textup{(1)}] $\E[f\mid\B]$ is a conditional expectation of $f$ given $\B$ in the sense of Definition~\ref{def:condexp}.
 \item[\textup{(2)}] If $f$ is already $\B$-measurable and integrable, then $\E[f\mid\B] = f$ $\mu$-a.e. 
\end{enumerate}
\end{theorem}

\subsection{The tower property}\label{subsec:tower}

\begin{proposition}[Tower property {\cite[Eq.~(10.1.4)]{alma991018313432607596}}]
\label{prop:tower}
Let $\B_1\subset\B\subset\A$ be sub-$\sigma$-algebras. Then, for every $f\in L^1(\mu)$,
\begin{equation}\label{eq:tower}
\E[\E[f\mid\B]\mid\B_1] = \E[f\mid\B_1] = \E[\E[f\mid\B_1]\mid\B], \qquad \mu\text{-a.e.}
\end{equation}
\end{proposition}

\subsection{The Doob--Dynkin lemma}\label{subsec:doobdynkin}

The conditional expectation $\E[f\mid\eta]$ is defined abstractly as any $\sigma(\eta)$-measurable random variable satisfying the integral identity \(\int_B fd\mu = \int_B \E[f\mid\B]d\mu, \) for all \(B \in \B\). This means $\E[f\mid\eta]$ is determined only up to a set of measure zero, and \emph{a priori} it is not clear that it can be written as a deterministic function of the observed value of~$\eta$.
The Doob--Dynkin lemma resolves this by characterising \emph{exactly} which random variables are $\sigma(\eta)$-measurable (they are the Borel functions of~$\eta$).

\begin{theorem}[Doob--Dynkin {\cite[Thm.~2.12.3]{alma991018313432607596}}]
\label{thm:doob-dynkin}
Let $I$ be a countable index set and $\{f_i\}_{i\in I}$ be a family of measurable functions on a nonempty space $X$. A function $g$ on $X$ is measurable with respect to $\sigma(\{f_i\}_{i\in I})$ if and only if there exists a Borel-measurable function $\psi\colon\R^I\to\R$ such that 
\[
g(x) = \psi(f_{i_1}(x),f_{i_2}(x),\ldots).
\]
\end{theorem}

Since $\E[f\mid\eta]$ is $\sigma(\eta)$-measurable by definition, the theorem (applied to the single-function family $\{f_i\}=\{\eta\}$) guarantees the existence of a Borel-measurable function $h_f\colon Y\to\R$ such that
\[
\E[f\mid\eta](\omega) = h_f(\eta(\omega)), \qquad \PR\text{-a.e.}\ \omega.
\]
The conditional expectation, which is \emph{a priori} just some random variable on $\Omega$, is in fact a deterministic function of the observed value $\eta(\omega)$. This is what justifies the pointwise notation
\[
\E[f\mid\eta=x] := h_f(x),
\]
which evaluates the conditional expectation at a specific value $x\in Y$ of the conditioning variable. The function $h_f$ is unique up to redefinition on a set of $(\PR\circ\eta^{-1})$-measure zero.

\subsection{Transition probabilities (Markov kernels)}\label{subsec:kernel}

\begin{definition}[Transition probability {\cite[Def.~10.7.1]{alma991018313432607596}}]
\label{def:kernel}
Let $(X_1,\B_1)$ and $(X_2,\B_2)$ be measurable spaces. A \emph{transition probability} (or \emph{Markov kernel}) for this pair is a function $\kappa\colon \B_2\times X_1\to[0,1]$ such that:
\begin{enumerate}
 \item[\textup{(i)}] for every fixed $x\in X_1$, the map $B\mapsto \kappa(B\mid x)$ is a probability measure on $\B_2$;
 \item[\textup{(ii)}] for every fixed $B\in\B_2$, the map $x\mapsto \kappa(B\mid x)$ is $\B_1$-measurable.
\end{enumerate}
\end{definition}

\subsection{The kernel Fubini theorem}\label{subsec:kernelfubini}

\begin{theorem}[Kernel Fubini {\cite[Thm.~10.7.2]{alma991018313432607596}}]
\label{thm:kernel-fubini}
Let $\kappa$ be a transition probability for $(X_1,\B_1)$ and $(X_2,\B_2)$, and let $\nu$ be a probability measure on $\B_1$. Then there exists a unique probability measure $\mu$ on $(X_1\times X_2,\B_1\otimes\B_2)$ satisfying
\begin{equation}\label{eq:kernel-prod}
\mu(B_1\times B_2) = \int_{B_1} \kappa(B_2\mid x)\nu(dx), \qquad \forall B_1\in\B_1, B_2\in\B_2.
\end{equation}
Moreover, for every $f\in L^1(\mu)$, the iterated integral is well defined and
\begin{equation}\label{eq:kernel-fubini}
\int_{X_1\times X_2} f(x_1,x_2)\mu(d(x_1,x_2))
=
\int_{X_1}\left(\int_{X_2} f(x_1,x_2)\kappa(dx_2\mid x_1)\right)\nu(dx_1).
\end{equation}
\end{theorem}
\begin{remark}
    In the above, recall that \(\B_1 \otimes \B_2\) is the product $\sigma$-algebra generated by the collection of rectangles $\{B_1 \times B_2 : B_1 \in \B_1, B_2 \in \B_2\}$.
\end{remark}

\subsection{Existence of the conditional law as a transition probability}\label{subsec:rcd}

\begin{theorem}[Regular conditional distribution {\cite[Ex.~10.7.5]{alma991018313432607596}}]
\label{thm:rcd}
Let $(\Omega,\A,\PR)$ be a probability space and let $S$ be a Souslin space (that is, $S$ is the image of a Polish (complete, separable, metrisable) space under a continuous mapping) and let $(Y,\A_Y)$ be a measurable space, and let
\[
\xi\colon(\Omega,\A)\to(S,\B(S)),\qquad
\eta\colon(\Omega,\A)\to(Y,\A_Y)
\]
be measurable mappings (i.e., $\xi$ and $\eta$ are random variables taking values in $S$ and $Y$, respectively). Then there exists a transition probability (Markov kernel)
\[
\kappa\colon \B(S)\times Y\to[0,1],\qquad (B,y)\longmapsto\kappa(B\mid y),
\]
such that for every $B\in\B(S)$,
\begin{equation}\label{eq:rcd}
\PR(\xi\in B\mid\eta) = \kappa(B\mid\eta), \qquad \PR\text{-a.e.}
\end{equation}
Moreover, the family $\{\kappa(\cdot\mid y)\}_{y\in Y}$ is unique up to modification on a set of $(\PR\circ\eta^{-1})$-measure zero.
\end{theorem}

\subsection{Agreement of the kernel and $\sigma$-algebra definitions}\label{subsec:agreement}

The $\sigma$-algebra formalism produces the conditional expectation $\E[f(\xi)\mid\eta]$ as an abstract $\sigma(\eta)$-measurable random variable, determined only $\PR$-a.e. The Doob--Dynkin lemma (Theorem~\ref{thm:doob-dynkin}) guarantees a Borel function $h_f$ with $\E[f(\xi)\mid\eta](\omega)=h_f(\eta(\omega))$, and one \emph{defines} $\E[f(\xi)\mid\eta=y]:=h_f(y)$. On the other hand, the kernel from Theorem~\ref{thm:rcd} gives a pointwise formula $y\mapsto\int_S fd\kappa(\cdot\mid y)$. The following proposition, which is a special case of a general result connecting regular conditional measures to conditional expectations \cite[Prop.~10.4.18]{alma991018313432607596}, shows that the two coincide.

\begin{proposition}[Kernel representation of conditional expectation {\cite[Prop.~10.4.18, Eq.~(10.4.13)]{alma991018313432607596}}]
\label{prop:kernel-sigma-agree}
In the setting of Theorem~\ref{thm:rcd}, for every $f\in L^1(\PR)$ of the form $f=\varphi\circ\xi$ with $\varphi$ a Borel function on $S$, one has
\begin{equation}\label{eq:condexp-kernel}
\E[\varphi(\xi)\mid\eta=y] = \int_S \varphi(s)\kappa(ds\mid y), \qquad (\PR\circ\eta^{-1})\text{-a.e.\ } y.
\end{equation}
\end{proposition}

This is a consequence of the integral identity for regular conditional measures \cite[Thm.~10.4.5]{alma991018313432607596}. 

\subsection{Pushforward and change of variables}\label{subsec:pushforward}

\begin{theorem}[Change of variables / pushforward {\cite[Thm.~3.6.1]{alma991018313432607596}}]
\label{thm:pushforward}
Let $\mu$ be a nonnegative measure on $(X,\A)$ and $\Phi\colon X\to Y$ an $(\A,\B)$-measurable mapping. A $\B$-measurable function $g$ on $Y$ is integrable with respect to the image (pushforward) measure $\Phi_\#\mu := \mu\circ\Phi^{-1}$ if and only if $g\circ\Phi$ is $\mu$-integrable. In that case,
\begin{equation}\label{eq:pushforward}
\int_Y g(y)(\Phi_\#\mu)(dy) = \int_X g(\Phi(x))\mu(dx).
\end{equation}
\end{theorem}

\subsection{Disintegration of an integral via a Markov kernel}\label{subsec:disintegration}

Combining the preceding results yields the following identity, which is the key computational device in the proof of Theorem~\ref{thrm:pushforwardkfe}. Suppose $\xi\colon\Omega\to S$ is a random variable in a Souslin space, $\eta:=\Phi\circ\xi\colon\Omega\to\cX$ for a Borel map $\Phi$, and $\kappa$ is the regular conditional distribution of $\xi$ given $\eta$ (Theorem~\ref{thm:rcd}). Then for any $\PR$-integrable function of the form $\varphi(\xi)$, the pushforward formula (Theorem~\ref{thm:pushforward}) and the kernel Fubini theorem (Theorem~\ref{thm:kernel-fubini}) give
\begin{equation}\label{eq:disintegrate}
\int_\Omega \varphi(\xi)d\PR
=\int_\cX\left(\int_S \varphi(y)\kappa(dy\mid x)\right)(\Phi_\#p_\xi)(dx),
\end{equation}
where $p_\xi:=\PR\circ\xi^{-1}$ denotes the law of $\xi$.

We now specialise the above to the setting of Theorem~\ref{thrm:pushforwardkfe}.

\begin{corollary}[Conditional law of $Y_t$ given $\Phi(Y_t)$]\label{cor:conditional-law}
Let $(\cY,d_{\cY})$ be a Polish space, let $\Phi\colon\cY\to\cX$ be a $(\B(\cY),\B(\cX))$-measurable map, and let $(Y_t)_{0\le t\le 1}$ be a stochastic process on $\cY$ with time-marginals $(p_t^{\cY})_{0\le t\le 1}$. For each $t\in[0,1]$, define $p_t^{\cX}:=\Phi_\# p_t^{\cY}$. Then:
\begin{enumerate}[label=\textup{(\roman*)}]
\item Since $\cY$ is Polish (hence Souslin), Theorem~\ref{thm:rcd} provides a Markov kernel $\kappa_t\colon\B(\cY)\times\cX\to[0,1]$ such that
\[
\PR(Y_t \in B \mid \Phi(Y_t)) = \kappa_t(B \mid \Phi(Y_t)), \qquad \PR\text{-a.e.}, \quad \forall B \in \B(\cY).
\]
We denote this kernel by
\begin{equation}\label{eq:conditional-law-notation}
p_t^{\cY}(dy_t \mid x_t) := \kappa_t(dy_t \mid x_t),
\end{equation}
so that $p_t^{\cY}(\cdot \mid \Phi(Y_t) = x_t)$ is the conditional law of $Y_t$ given $\Phi(Y_t) = x_t$. This can also be written \[p_t^{\cY}(dy_t \mid \Phi(Y_t) = x_t) := p_t^{\cY}(dy \mid x_t). \]
\item For any $p_t^{\cY}$-integrable function $\varphi\colon\cY\to\R$, the disintegration identity in~\eqref{eq:disintegrate} gives
\begin{equation}\label{eq:disintegrate-specialised}
\int_{\cY} \varphi(y_t) p_t^{\cY}(dy_t) = \int_{\cX}\left(\int_{\cY} \varphi(y_t) p_t^{\cY}(dy_t \mid x_t)\right) p_t^{\cX}(dx_t).
\end{equation}
\item By Proposition~\ref{prop:kernel-sigma-agree}, the kernel integral and the conditional expectation agree:
\begin{equation}\label{eq:kernel-condexp-agree}
\int_{\cY} \varphi(y_t) p_t^{\cY}(dy_t \mid x_t) = \E[\varphi(Y_t) \mid \Phi(Y_t) = x_t], \qquad p_t^{\cX}\text{-a.e.}\ x_t.
\end{equation}
\end{enumerate}
\end{corollary}

\subsection{Hypotheses}\label{subsec:hypotheses}

We collect here the standing assumptions for Theorem~\ref{thrm:pushforwardkfe}. We first fix notation for the relevant function spaces.

\begin{definition}[Function spaces]\label{def:functionspaces}
Let $(S,d)$ be a metric space.
\begin{enumerate}[label=\textup{(\roman*)},leftmargin=2em]
\item $C_0(S)$ denotes the space of continuous functions $f\colon S\to\R$ that \emph{vanish at infinity}: for every $\varepsilon>0$ there exists a compact set $K\subset S$ such that $|f(s)|<\varepsilon$ for all $s\notin K$. When $S$ is compact, $C_0(S)=C(S)$.
\item For $k\in\{1,2,\ldots,\infty\}$ and $S$ a smooth manifold (possibly with boundary), $C_0^k(S)$ denotes the space of $C^k$ functions that, together with all derivatives up to order~$k$, vanish at infinity. When $S$ is compact, $C_0^k(S)=C^k(S)$.
\end{enumerate}
\end{definition}

\begin{assumption}[Generator Matching on $\cY$]\label{ass:gm-cY}
$(\cY,d_{\cY})$ is a Polish space. $(Y_t)_{0\le t\le 1}$ is a time-inhomogeneous Feller process on $\cY$ with infinitesimal generator $W_t$, time-marginals $(p_t^{\cY})_{0\le t\le 1}$, and test-function space $\cD(W_t)\subset C_0(\cY)$, satisfying the Generator Matching regularity conditions of~\cite[Appendix~A.2, Assumptions~1--5]{holderrieth2025generatormatchinggenerativemodeling}.
\end{assumption}

\begin{assumption}[Pushforward map]\label{ass:Phi}
$(\cX,\B(\cX))$ is a measurable space equipped with a measure-determining linear space of bounded test functions $\cT(\cX)$. The map $\Phi\colon\cY\to\cX$ is $(\B(\cY),\B(\cX))$-measurable, and for every $f\in\cT(\cX)$ and every $t\in[0,1)$:
\begin{enumerate}[label=\textup{(\roman*)},leftmargin=2em]
 \item $f\circ\Phi \in \cD(W_t)$;
 \item the map $t\mapsto W_t(f\circ\Phi)$ is continuous in $\lVert\cdot\rVert_\infty$ on $[0,1)$.
\end{enumerate}
\end{assumption}
\begin{remark}
    Assumption \ref{ass:Phi} places conditions on the projection $\Phi$ from \(\mathcal Y\) to \(\mathcal X\), and it does not restrict the choice of neural network architecture or loss.
\end{remark}
\begin{assumption}[Integrability]\label{ass:integrability}
For every $f\in\cT(\cX)$ and every $t\in[0,1)$, the generator's action on the lifted test function is integrable:
\begin{equation}\label{eq:integrability}
\E_{y_t \sim p_t^{\cY}}[\lvert W_t(f\circ\Phi)(y_t)\rvert] < \infty.
\end{equation}
If the conditional generator admits a linear parametrisation (Definition~\ref{def:condlinparam}) with target $F_t^{y_t}(\Phi(y_t))$, we additionally require
\begin{equation}\label{eq:integrability-target}
\E_{y_t \sim p_t^{\cY}}[\lVert F_t^{y_t}(\Phi(y_t))\rVert_{V_{t,\Phi(y_t)}}] < \infty, \qquad t\in[0,1).
\end{equation}
\end{assumption}

\begin{assumption}[KFE sufficiency on $\cX$]\label{ass:gm-cX}
The KFE with the marginal generator $L_tf$ of Theorem~\ref{thrm:pushforwardkfe} uniquely determines the probability path on~$\cX$: if $(q_t)_{0\le t\le 1}$ is any probability path with $q_0=p_0^{\cX}:=\Phi_\# p_0^{\cY}$ and $\partial_t\langle q_t,f\rangle=\langle q_t,L_t f\rangle$ for all $f\in\cT(\cX)$ and $t\in[0,1)$, then $q_t=\Phi_\# p_t^{\cY}$ for all $t\in[0,1]$.
\end{assumption}

We remark that Assumption~\ref{ass:gm-cX} is non-trivial and we give an example of it being violated in \S\ref{subs:regularity-violating}, despite $\Phi \in C^\infty$ and $Y_t$ simply being a Brownian motion. Informally, KFE insufficiency may occur when the first-order dynamics of the process (such as drift, diffusion, and instantaneous rate of jumps in the Euclidean case) do not uniquely determine the marginal dynamics --- that is, when there exists a stochastic process whose instantaneous dynamics are indistinguishable from the target process in $\mathcal{X}$ but which does not have the desired marginal distribution. This is exactly the failure mode of Example~\ref{subs:regularity-violating}. We note, however, that this hypothesis is agnostic to the dimensionality of $\mathcal{X}$ and $\mathcal{Y}$. For further discussion of uniqueness of the KFE, we direct the reader to Generator Matching~\citep{holderrieth2025generatormatchinggenerativemodeling}, Appendix~A.2.

\paragraph{Discussion of the hypotheses.}\label{rmk:hypotheses}
Assumption~\ref{ass:gm-cY} is simply the statement that $(Y_t, W_t)$ is a Generator Matching process on~$\cY$. The proof of Theorem~\ref{thrm:pushforwardkfe} uses Assumptions~\ref{ass:gm-cY}--\ref{ass:integrability}. Assumption~\ref{ass:gm-cX} is not needed for the KFE itself but is required in order to \emph{use} the marginal generator $L_t$ within the Generator Matching framework.

\paragraph{Domain compatibility.}
Since $\cD(W_t)\subset C_0(\cY)$, condition~(i) requires both that $f\circ\Phi$ vanish at infinity and that $f\circ\Phi$ belong to the domain of $W_t$, which may impose further regularity constraints on $\Phi$, e.g., differentiability. 

\paragraph{Compactness.} When $\cY$ is compact, $C_0(\cY) = C(\cY)$ and $C_0^k(\cY) = C^k(\cY)$, so the vanishing-at-infinity requirement on $f\circ\Phi$ is vacuous The remaining question is whether $f\circ\Phi$ has sufficient regularity to belong to $\cD(W_t)$. For many standard generators, $\cD(W_t) =  C_0^k(\cY) = C^k(\cY)$ for some $k$ (e.g.\ $k=2$ for second-order differential operators, $k=0$ for rate matrices), so it suffices that $\Phi$ and $f$ are $C^k$. Moreover, integrability (Assumption~\ref{ass:integrability}) is automatic on a compact space, since continuous functions are bounded and integration against a probability measure is finite. The most important special case in which compactness simplifies the verification is the product-space setting $\cY = \cX \times \cZ$, $\Phi = \pi_\cX$ (see \S\ref{subsec:latentprocess}). When $\cZ$ is compact, the domain compatibility and integrability conditions along the $\cZ$-fibre are automatic for continuous integrands. This covers, for example, the case of a continuous main process $X_t$ jointly evolving with a finite-state latent CTMC $Z_t$ (as in the example of \S\ref{subs:1d-example}), since any finite set is compact. It also covers flow matching and diffusion on compact Riemannian manifolds, such as $\operatorname{SO}(3)$.

\paragraph{Non-compact state spaces.}
When $\cY$ is non-compact, condition~(i) requires $f\circ\Phi\in\cD(W_t)\subset C_0(\cY)$, so in particular $f\circ\Phi$ must vanish at infinity. This is a genuine restriction on~$\Phi$. It holds whenever $\Phi$ is proper (preimages of compact sets are compact). However, it may fail for general maps such as projections $\pi_\cX : \cX\times \cZ\to \cX$ when $\cZ$ is non-compact.

\subsection{Proof of Theorem~\ref{thrm:pushforwardkfe}}\label{subsec:proof}
\begin{proof}
By the change-of-variables formula for the pushforward (Theorem~\ref{thm:pushforward}),
\begin{equation}\label{eq:pf-step1}
\langle p_t^{\cX}, f\rangle
= \int_\cX f(x_t)(\Phi_\# p_t^{\cY})(dx_t)
= \int_{\cY} f(\Phi(y_t)) p_t^{\cY}(dy_t)
= \langle p_t^{\cY}, f\circ\Phi\rangle.
\end{equation}
By Assumption~\ref{ass:Phi}(i), $f\circ\Phi\in\cD(W_t)$. Since $Y_t$ is generated by $W_t$ and has time-marginals $p_t^{\cY}$, the KFE holds
\begin{equation}\label{eq:pf-step2}
\partial_t\langle p_t^{\cY}, f\circ\Phi\rangle
= \langle p_t^{\cY}, W_t(f\circ\Phi)\rangle
= \int_{\cY} W_t(f\circ\Phi)(y_t)p_t^{\cY}(dy_t).
\end{equation}
Applying Corollary~\ref{cor:conditional-law}(ii) with $\varphi := W_t(f\circ\Phi)$, noting that $\varphi$ is $p_t^{\mathcal Y}$-integrable by Assumption~\ref{ass:integrability}, we have
\begin{equation}\label{eq:pf-step3}
\int_{\cY} W_t(f\circ\Phi)(y_t)p_t^{\cY}(dy_t)
= \int_{\cX}\left(\int_{\cY} W_t(f\circ\Phi)(y_t)p_t^{\cY}(dy_t\mid x_t)\right) p_t^{\cX}(dx_t).
\end{equation}
By Corollary~\ref{cor:conditional-law}(iii), the inner integral in~\eqref{eq:pf-step3} is precisely the conditional expectation evaluated at the point $x_t$:
\begin{equation}\label{eq:pf-step4}
\int_{\cY} W_t(f\circ\Phi)(y_t)p_t^{\mathcal Y}(dy_t\mid x_t)
= \E[W_t(f\circ\Phi)(Y_t)\mid\Phi(Y_t)=x_t]
= L_t f(x_t).
\end{equation}
Combining equations~\ref{eq:pf-step1}--\ref{eq:pf-step4}:
\[
\partial_t\langle p_t^{\cX}, f\rangle
= \int_{\cX} L_t f(x_t)p_t^{\cX}(dx_t)
= \langle p_t^{\cX}, L_t f\rangle. \qedhere
\]
\end{proof}
\subsubsection{Alternative proof via the tower property.}
We also remark that the tower property (Proposition~\ref{prop:tower}) gives a shorter argument that avoids explicit use of kernels. Start from the weak KFE for $Y_t$ evaluated against $f\circ\Phi\in\mathcal{T}^W$:
\[
\partial_t \mathbb E[(f\circ\Phi)(Y_t)]
= \mathbb E[W_t(f\circ\Phi)(Y_t)]
\]
The left-hand side equals $\partial_t\int f \, dp_t^\mathcal X$, since $p_t^{\mathcal{X}}=\Phi_\# p_t^{\mathcal{Y}}$. For the right-hand side, apply the tower property 
\begin{align*}
\mathbb{E}[W_t(f\circ\Phi)(Y_t)] &= \mathbb{E}[\mathbb{E}[W_t(f\circ\Phi)(Y_t)\mid\Phi(Y_t)]] \\
&= \mathbb{E}[\mathcal{L}_t f(\Phi(Y_t))] \\
&= \int \mathcal{L}_t f\,dp_t^{\mathcal{X}}.
\end{align*}
Combining both sides yields $\partial_t\langle p_t^\mathcal X, f\rangle  = \langle p_t^\mathcal X, \mathcal L_t f\rangle$.
\section{Gradients are preserved}
\subsection{A property of Bregman Divergences}
Here we restate a well-known property of Bregman divergences, seen in e.g.  \citep{lipman2024flowmatchingguidecode, billera2025timedependentlossreweighting}.
\begin{lemma}\label{lem:bregmangradlemma}
    (Expectations commute with Bregman divergences under gradients) Let $D_\phi(a,b) = \phi(a) - \phi(b) - \inner{a-b}{\nabla\phi(b)}$ be a Bregman divergence, where $\phi \colon \Omega_\phi \subset V \to \R$ is strictly convex and differentiable on the closed and convex domain $\Omega_\phi \subset V$. Let $X$ be an integrable $\Omega_\phi$-valued random variable such that $\E[|\phi(X)|] < \infty$, and let $f_\theta(x) \in \Omega_\phi$. Then
    \begin{equation*}
        \nabla_\theta \E[D_\phi(X,f_\theta(x))] = \nabla_\theta D_\phi(\E[X],f_\theta(x)).
    \end{equation*}
\end{lemma}
\begin{proof}
    Write $b := f_\theta(x)$ for brevity. Since $\Omega_\phi$ is closed and convex and $X$ is an integrable $\Omega_\phi$-valued random variable, one has $\E[X] \in \Omega_\phi$, so $D_\phi(\E[X],b)$ is well-defined. We also remark that the inner product term $\inner{X-b}{\nabla\phi(b)}$ has finite expectation since $\nabla\phi(b)$ is deterministic and the Cauchy--Schwarz inequality gives $\E[|\inner{X}{\nabla\phi(b)}|] \leq \lVert\nabla\phi(b)\rVert \E[\lVert X\rVert] < \infty$.

    Now, note that, expanding the Bregman divergence and taking expectations, we obtain
    \begin{align*}
        \E[D_\phi(X,b)]
        &= \E[\phi(X)] - \phi(b) - \E[\inner{X - b}{\nabla\phi(b)}]\\
        &= \E[\phi(X)] - \phi(b) - \inner{\E[X] - b}{\nabla\phi(b)}.
    \end{align*}

    On the other hand, expanding $D_\phi(\E[X],b)$ gives
    \begin{equation*}
        D_\phi(\E[X],b) = \phi(\E[X]) - \phi(b) - \inner{\E[X] - b}{\nabla\phi(b)}.
    \end{equation*}
    Comparing the two expressions yields
    \begin{equation}\label{eq:bregman-remainder}
        \E[D_\phi(X,b)] = D_\phi(\E[X],b) + \E[\phi(X)] - \phi(\E[X]).
    \end{equation}
    The remainder $\E[\phi(X)] - \phi(\E[X])$ is independent of $\theta$. Applying $\nabla_\theta$ to both sides eliminates this constant and one has
    \begin{equation*}
        \nabla_\theta \E[D_\phi(X,f_\theta(x))] = \nabla_\theta D_\phi(\E[X],f_\theta(x)). \qedhere
    \end{equation*}
\end{proof}

\subsection{Proof of Theorem~\ref{thrm:losseshavesamegrads}}\label{sec:gradientproof}
\begin{proof}
We show that $\nabla_\theta L_{\mathrm{cgm}}(\theta) = \nabla_\theta L_{\mathrm{gm}}(\theta)$.  Writing the conditional loss in its disintegrated form (Definition~\ref{def:condgmloss}), we compute
\begin{align*}
\nabla_\theta
& L_{\mathrm{cgm}}(\theta) \\
&= \nabla_\theta
\E_{t \sim U[0,1]}\E_{x_t \sim p_t^{\mathcal{X}}}\E_{y_t \sim p_t^{\mathcal{Y}}(\cdot\mid\Phi(Y_t)=x_t)}
  [D_{t,x_t}(F_t^{y_t}(x_t), F_t^\theta(x_t))] \\
&= 
\E_{t \sim U[0,1]}\E_{x_t \sim p_t^{\mathcal{X}}}
\nabla_\theta \E_{y_t \sim p_t^{\mathcal{Y}}(\cdot\mid\Phi(Y_t)=x_t)}
  [D_{t,x_t}(F_t^{y_t}(x_t), F_t^\theta(x_t))]
&& \text{} \\
&= 
\E_{t \sim U[0,1]}\E_{x_t \sim p_t^{\mathcal{X}}}
\nabla_\theta D_{t,x_t}(\E_{y_t \sim p_t^{\mathcal{Y}}(\cdot\mid\Phi(Y_t)=x_t)}[F_t^{y_t}(x_t)], F_t^\theta(x_t))
&& \text{(Lemma~\ref{lem:bregmangradlemma})} \\
&= 
\E_{t \sim U[0,1]}\E_{x_t \sim p_t^{\mathcal{X}}}
\nabla_\theta D_{t,x_t}(F_t(x_t), F_t^\theta(x_t))
&& \text{(Eq.~\ref{eqn:margfromcond})} \\
&= 
\nabla_\theta \E_{t \sim U[0,1]}\E_{x_t \sim p_t^{\mathcal{X}}}
  [D_{t,x_t}(F_t(x_t), F_t^\theta(x_t))]
&& \text{} \\
&= \nabla_\theta L_{\mathrm{gm}}(\theta).
\end{align*}
In the third line, we apply Lemma~\ref{lem:bregmangradlemma} to the inner expectation over $y_t$, noting that the the random variable corresponds to $F_t^{y_t}(x_t) \in \Omega_{t,x_t}$ and the second argument $F_t^\theta(x_t)$ depends on $\theta$ but not on $y_t$, so the Bregman gradient-expectation interchange applies. In the fourth line, we use the identity $F_t(x_t) = \E[F_t^{Y_t}(x_t) \mid \Phi(Y_t) = x_t]$ from~\eqref{eqn:margfromcond}.
\end{proof}
\section{One-dimensional example}\label{app:mixedconddiscrete}

This appendix provides a derivation for the one-dimensional example of \S\ref{subs:1d-example}. The example is an instance of the double conditioning of \S\ref{subsec:doubleconditioning}, in that the static latent variable is the pair of endpoints \((x_0,x_1)\), and the latent stochastic process is the CTMC state \(z_t\). We also note that in our empirical example, the rate $\lambda_z^{x_1}(x,t)$ does not depend on the endpoint $x_1$, but we present the general case. 

\subsection{The conditional generator}\label{subsec:mixedcondgen}

Fix endpoints \(x_0, x_1 \in \R\). We define the joint process \((X_t,Z_t)\) on \(\R \times \{-1,+1\}\), conditioned on \(x_0,x_1\), by specifying its infinitesimal generator and initial distribution $(X_0,Z_0) \sim (\delta_{x_0}, \mathrm{Unif}(\{-1,1\}))$.

For test functions \(g \in C^2(\R) \otimes C(\{-1,+1\})\), let
\begin{equation}\label{eqn:mixedgen}
    W_t^{x_1} g(x,z) = b_t^{z,x_1}(x) \partial_x g(x,z) + \frac{1}{2}\sigma_t^2 \partial_{xx} g(x,z) + \lambda_z^{x_1}(x,t)[g(x,-z) - g(x,z)],
\end{equation}
where \(b_t^{z,x_1}(x) := \frac{z x_1 - x}{1-t}\) and \(\lambda_z^{x_1}(x,t) \geq 0\) is the rate of jumping from \(z\) to \(-z\). Conditional on \(Z_t = z\) and $x_1$, the continuous component $X_t$ evolves as a Brownian bridge
\begin{equation*}
    dX_t = \frac{Z_t x_1 - X_t}{1-t} dt + \sigma_t dB_t,
\end{equation*}
directed towards \(x_1\) when \(Z_t = +1\) and towards \(-x_1\) when \(Z_t = -1\), while the discrete component jumps between \(\pm 1\) at state-dependent rate \(\lambda_z^{x_1}(X_t,t)\). The rates \(\lambda_z^{x_1}\) are chosen so that \(X_1 = x_1\) a.s.\ in the conditional paths. The resulting conditional trajectories are illustrated in subfigure~B of Figure~\ref{fig:switchingdriftexample}.

Applying \eqref{eqn:mixedgen} to \(f \circ \pi_\cX\), the CTMC term vanishes, giving the conditional generator
\begin{equation}\label{eqn:switchcondgen}
    L_t^{x_1,z}f(x_t) = \frac{z_t x_1 - x_t}{1-t}f'(x_t) + \frac{1}{2}\sigma_t^2 f''(x_t).
\end{equation}
Rather than parametrising the velocity \(b_t^{z,x_1}(x_t) = \frac{z_t x_1 - x_t}{1-t}\) directly (which would require training against a quantity that diverges as \(t \to 1\)) we adopt an \(x_1\)\textit{-prediction} parametrisation. This admits the linear parametrisation \[L_t^{x_1,z}f(x_t) = \langle \mathcal{K}_{t,x_t}f, F_t^{x_1,z}(x_t)\rangle_{V_{t,x_t}}\] with \(\Omega_{t,x_t} = V_{t,x_t} = \R^2\), equipped with the Euclidean inner product, and
\begin{equation}\label{eqn:x1predlinparam}
    \mathcal{K}_{t,x_t}f = \begin{pmatrix} \frac{f'(x_t)}{1-t} \\ -\frac{x_t f'(x_t)}{1-t} + \frac{1}{2}\sigma_t^2 f''(x_t) \end{pmatrix}, \qquad F_t^{x_1,z}(x_t) = \begin{pmatrix} z_t x_1 \\ 1 \end{pmatrix}.
\end{equation}
Note that \(\mathcal{K}_{t,x_t}\) depends on the joint state \((x_t,z_t)\) only through \(x_t\), so this is a valid conditional linear parametrisation in the sense of Definition~\ref{def:condlinparam}.

\subsection{Training}\label{subsec:mixedtraining}

By \eqref{eqn:margfromcond}, the marginal target is
\begin{equation*}
    F_t(x_t) = \E[F_t^{x_1,Z}(x_t) \mid X_t = x_t] = \begin{pmatrix} \E[Z_t \mid X_t = x_t]\cdot x_1 \\ 1 \end{pmatrix}.
\end{equation*}
The conditional dependence on \(z_t\) has been integrated out, and the first component depends on \(x_t\) through the conditional mean \(\E[Z_t \mid X_t = x_t]\), while the second component is constant.  By Theorem~\ref{thrm:losseshavesamegrads}, a neural network \(F_t^\theta(x_t) = (x_1^\theta(x_t), f_\theta^{(2)}(x_t))^T\) that receives only \(x_t\) as input may be trained against the conditional loss
\begin{equation*}
    L_{\mathrm{cgm}}(\theta) = \E_{t \sim U[0,1], (x_0,x_1) \sim q, (x_t, z_t) \sim p_t(dx,dz \mid x_0,x_1)}\left[\left\| F_t^{x_1,z}(x_t) - F_t^\theta(x_t)\right\|^2\right]
\end{equation*}
and still recover the correct marginal target \(F_t(x_t)\). Since \(F_t^{x_1,z}(x_t) = (z_t x_1, 1)^T\), the second component is minimised trivially by \(f_\theta^{(2)} \equiv 1\) and may be fixed a priori, so it suffices to train only \(x_1^\theta\) against
\begin{equation*}
    \E_{t \sim U[0,1], (x_0,x_1) \sim q, (x_t, z_t) \sim p_t(\cdot \mid x_0,x_1)}\left[\left\| z_t x_1 - x_1^\theta(x_t)\right\|^2\right].
\end{equation*}

\section{Edit Flows as a special case}\label{app:editflows}

In this appendix we give the detailed derivation showing that Theorem~3.1 of \citet{havasi2025editflowsflowmatching} is recovered as a special case of the pushforward conditional generator matching framework of \S\ref{subsec:pushforwardcgm}--\S\ref{subsec:latentprocess}. We assume that the state space is finite, which has no practical consequence but simplifies the test function space to $C(\mathcal X \times \mathcal Z) \cong \mathbb R^{|\mathcal X \times \mathcal Z|}$.

\subsection{Conditional generator and linear parametrisation}\label{subsec:ef-condgen}

Let \(\cX\) and \(\cZ\) be finite sets and let \((X_t,Z_t)\) be a CTMC on the joint state space \(\cX \times \cZ\) with rates \(u_t(x',z' \mid x_t, z_t) \geq 0\). Then the infinitesimal generator on the joint space acts on test functions \(g \in C(\cX \times \cZ)\) as
\begin{equation}\label{eqn:jointctmcgen}
    W_t g(x_t,z_t) = \sum_{(x',z')} u_t(x',z' \mid x_t,z_t)[g(x',z') - g(x_t,z_t)].
\end{equation}
We verify that the conditional generator \(L_t^{z_t}\) of \S\ref{subsec:latentprocess} admits a linear parametrisation in the sense of Definition~\ref{def:condlinparam} in which \(\mathcal{K}_{t,x_t}\) and \(V_{t,x_t}\) depend only on \(x_t\). Applying \eqref{eqn:jointctmcgen} to a test function of the form \(f \circ \pi_{\cX}\), we obtain
\begin{equation}\label{eqn:efcondgen}
    L_t^{z_t}f(x_t) = W_t(f \circ \pi_\cX)(x_t,z_t) = \sum_{x'}[f(x') - f(x_t)]\underbrace{\sum_{z'}u_t(x',z' \mid x_t,z_t)}_{=: \tilde{u}_t(x' \mid x_t, z_t)}.
\end{equation}
The quantity \(\tilde{u}_t(x' \mid x_t,z_t) := \sum_{z'} u_t(x',z' \mid x_t,z_t)\) is the total rate of the \(\cX\)-component jumping to \(x'\), given the current joint state \((x_t,z_t)\), irrespective of where \(z\) transitions. Taking \(V_{t,x_t} := \R^{|\cX|}\) with the standard inner product and defining
\begin{equation*}
    \mathcal{K}_{t,x_t} f := (f(x') - f(x_t))_{x' \in \cX} \in \R^{|\cX|}, \qquad F_t^{z_t}(x_t) := (\tilde{u}_t(x' \mid x_t, z_t))_{x' \in \cX} \in \R_{\geq 0}^{|\cX|},
\end{equation*}
the conditional generator has the linear parametrisation
\begin{equation*}
    L_t^{z_t}f(x_t) = \inner{\mathcal{K}_{t,x_t}f}{F_t^{z_t}(x_t)}_{V_{t,x_t}}.
\end{equation*}
Both \(\mathcal{K}_{t,x_t}\) and \(V_{t,x_t}\) depend on the current state \((x_t,z_t)\) only through \(x_t\), so this is a valid conditional linear parametrisation in the sense of Definition~\ref{def:condlinparam}. The \(z_t\)-dependence is confined entirely to the training target \(F_t^{z_t}(x_t)\).

\subsection{Recovery of Theorem~3.1 of \citet{havasi2025editflowsflowmatching}}\label{subsec:ef-recovery}

By \eqref{eqn:margfromcond} and the linearity of the inner product in the second argument, the linear parametrisation of the marginal generator has
\begin{equation*}
    F_t(x_t) = \E[F_t^{Z_t}(x_t) \mid X_t = x_t],
\end{equation*}
whose \(x'\)-component is
\begin{equation*}
    F_t(x_t)_{x'} = \E_{z_t \sim p_t(\cdot \mid x_t)}\left[\sum_{z'} u_t(x',z' \mid x_t,z_t)\right] = \sum_{z'}\E_{z_t \sim p_t(\cdot \mid x_t)}[u_t(x',z' \mid x_t,z_t)].
\end{equation*}
Writing \(u_t(x' \mid x_t) := F_t(x_t)_{x'}\) for the projected marginal rate, we recover the first part of Theorem~3.1 of \citep{havasi2025editflowsflowmatching}. Indeed, as a direct consequence of Corollary~\ref{cor:projectedprocess}, if \(u_t(x',z' \mid x_t,z_t)\) generates \(p_t(x,z)\), then
\begin{equation}\label{eqn:efprojectedrates-app}
    u_t(x' \mid x_t) = \sum_{z'}\E_{z_t \sim p_t(\cdot \mid x_t)}[u_t(x',z' \mid x_t, z_t)] \qquad \text{generates} \qquad p_t^{\cX}(x) = \sum_z p_t(x,z).
\end{equation}

The second part of Theorem~3.1 in \citep{havasi2025editflowsflowmatching} states that if \(D_\phi\) is a Bregman divergence, then
\begin{equation*}
    \nabla_\theta \E_{(x_t,z_t) \sim p_t} [D_\phi(\tilde{u}_t(\cdot \mid x_t, z_t), u_t^\theta(\cdot \mid x_t))] = \nabla_\theta \E_{x_t \sim p_t(x)}[D_\phi(u_t(\cdot \mid x_t), u_t^\theta(\cdot \mid x_t))].
\end{equation*}
In our setting, this is Theorem~\ref{thrm:losseshavesamegrads} applied to the conditional linear parametrisation above.

\section{Non-uniqueness of the KFE in Example~\ref{ex:regularity-violating}}\label{app:regularity-violating}

In this appendix we prove the claim made in Example~\ref{ex:regularity-violating}, namely that the marginal generator $L_t$ arising from applying the map $\Phi(x) = e^{-1/x}\cdot I_{\{x>0\}}$ to a standard Brownian motion does not uniquely determine the probability path, so that Assumption~\ref{ass:gm-cX} fails. For this, we construct a one-parameter family of probability paths all satisfying the same KFE with the same initial condition.

\subsection{The pushforward generator}\label{subsec:pfgen-computation}

Recall that $\Phi\colon\R\to\R_{\geq 0}$ is defined by $\Phi(x) = e^{-1/x}$ for $x>0$ and $\Phi(x)=0$ otherwise. Let $(B_t)_{t\geq 0}$ be a standard Brownian motion with generator $W_tf(x) = \frac{1}{2}\partial_{xx}f(x)$. For $t>0$, the pushforward marginals are
\begin{equation}\label{eqn:marginal-decomp}
    p_t^{\cX}(dx) = \frac{1}{2}\delta_0(dx) + \rho_t(x)dx,
\end{equation}
where $\rho_t$ is the density on $(0,1)$ arising from the pushforward of the positive half of the Gaussian.

For $x_t \in (0,1)$, the preimage $\Phi^{-1}(\{x_t\})$ is a single point $w = -1/\log x_t > 0$, so the conditional expectation in Theorem~\ref{thrm:pushforwardkfe} simply becomes 
\begin{align*}\label{eqn:interior-generator}
    L_t f(x_t) =  \mathbb E[W_t(f\circ \Phi)(B_t)\mid \Phi(B_t) = x_t] &=  \frac{1}{2}(f\circ\Phi)''(w) \\
    &= \frac{1}{2}\Phi'(w)^2 f''(x_t) + \frac{1}{2}\Phi''(w)f'(x_t).
\end{align*}
Moreover, switching to $x_t =0$, since $\Phi$ and all its derivatives vanish at $0$, the chain rule gives $(f\circ\Phi)^{(k)}(0) = 0$ for all $k\geq 1$. Therefore
\begin{equation}\label{eqn:boundary-generator}
    L_t f(0) = \E\bigl[\frac{1}{2}(f\circ\Phi)''(B_t) \mid B_t \leq 0\bigr] = 0
\end{equation}
for every $f\in\cT(\R_{\geq 0})$ and every $t > 0$.
\subsection{A one-parameter family of solutions}\label{subsec:nonuniqueness}
By the KFE, it holds that $\partial_t\langle p_t^{\cX}, f\rangle = \langle p_t^{\cX}, L_t f\rangle$, which in our case becomes
\begin{equation}\label{eqn:density-KFE}
    \int_0^1 f(x)\partial_t\rho_t(x)dx = \frac 12  L_tf(0) + \int_0^1 L_tf(x)\rho_t(x)dx = \int_0^1 L_tf(x) \rho_t(x)dx.
\end{equation}
For any $c\in[0,1]$, we define probability paths 
\begin{equation}\label{eqn:rescaled-path}
    q_t^{(c)}(dx) := \begin{cases} \delta_0(dx) & t = 0,\\ c\delta_0(dx) + 2(1-c)\rho_t(x)dx & t > 0.\end{cases}
\end{equation}
Since $B_0 = 0$ a.s., we have $p_0^{\cX} = \delta_0$, so every member of the one-parameter family shares the initial condition $q_0^{(c)} = p_0^\cX$. For $t > 0$ and $f\in \mathcal T(\mathbb R_{\geq 0})$, one has
\begin{align}
    \partial_t\langle q_t^{(c)}, f\rangle
    &= 2(1-c)\int_0^1 f(x)\partial_t\rho_t(x)dx
    = 2(1-c)\int_0^1 L_tf(x)\rho_t(x)dx, \label{eqn:lhs-kfe}\\[4pt]
    \langle q_t^{(c)}, L_t f\rangle
    &= cL_t f(0) + 2(1-c)\int_0^1 L_tf(x)\rho_t(x)dx = 2(1-c)\int_0^1 L_tf(x) \rho_t(x) dx.\label{eqn:rhs-kfe}
\end{align}
so $q_t^{(c)}$ satisfies $\partial_t\langle q_t^{(c)}, f\rangle = \langle q_t^{(c)}, L_t f\rangle$ for every $c\in[0,1]$. Therefore, in this case, Assumption~\ref{ass:gm-cX} fails.

\end{document}